\documentclass{article}

\usepackage{microtype}
\usepackage{graphicx}
\usepackage{subcaption}
\usepackage{booktabs} 
\usepackage{booktabs}
\usepackage{multirow}
\usepackage{tabularray}
\UseTblrLibrary{booktabs} 

\usepackage{hyperref}

\usepackage[preprint]{icml2026}

\usepackage{amsmath}
\usepackage{amssymb}
\usepackage{mathtools}
\usepackage{amsthm}
\usepackage{bbm}

\usepackage{algorithm}
\usepackage{algorithmic}
\newcommand{\TO}{\textbf{to}}
\usepackage[dvipsnames, svgnames]{xcolor}

\usepackage[capitalize,noabbrev]{cleveref}

\theoremstyle{plain}
\newtheorem{theorem}{Theorem}[section]
\newtheorem{proposition}[theorem]{Proposition}
\newtheorem{lemma}[theorem]{Lemma}

\theoremstyle{definition}

\theoremstyle{remark}

\usepackage[textsize=tiny]{todonotes}

\newtheorem{thm}{Theorem}
\newtheorem{lem}[thm]{Lemma} 
\newtheorem{prop}{Proposition}[section] 





\usepackage[textsize=tiny]{todonotes}

\icmltitlerunning{Preprint}

\begin{document}

\twocolumn[
  \icmltitle{SpikySpace: A Spiking State Space Model for Energy-Efficient Time Series Forecasting}



  \icmlsetsymbol{equal}{*}

  \begin{icmlauthorlist}
    \icmlauthor{Kaiwen Tang}{nus}
    \icmlauthor{Jiaqi Zheng}{sea}
    \icmlauthor{Yuze Jin}{nus}
    \icmlauthor{Yupeng Qiu}{nus}
    \icmlauthor{Guangda Sun}{nus}
    \icmlauthor{Zhanglu Yan}{nus}
    \icmlauthor{Weng-Fai Wong}{nus}
  \end{icmlauthorlist}

  \icmlaffiliation{nus}{School of Computing, National University of Singapore, Singapore, Singapore}
  \icmlaffiliation{sea}{Sea AI Lab, Singapore, Singapore}

  \icmlcorrespondingauthor{Zhanglu Yan}{zlyan@nus.edu.sg}

  \icmlkeywords{Machine Learning, ICML}

  \vskip 0.3in
]



\printAffiliationsAndNotice{}  

\begin{abstract}
 Time-series forecasting in domains like traffic management and industrial monitoring often requires real-time, energy-efficient processing on edge devices with limited resources. 
 Spiking neural networks (SNNs) offer event-driven computation and ultra-low power and have been proposed for use in this space.
 Unfortunately, existing SNN-based time-series forecasters often use complex transformer blocks. 
To address this issue, we propose SpikySpace, a spiking state-space model (SSM) that reduces the quadratic cost in the attention block to linear time via spiking selective scanning. 
Further, we introduce PTsoftplus and PTSiLU, two efficient approximations of SiLU and Softplus that replace costly exponential and division operations with simple bit-shifts. 
Evaluated on four multivariate time-series benchmarks, SpikySpace outperforms the leading SNN in terms of accuracy by up to 3.0\% while reducing energy consumption by over 96.1\%.
As the first fully spiking state-space model, SpikySpace bridges neuromorphic efficiency with modern sequence modeling, opening a practical path toward efficient time series forecasting systems. Our code is available at \href{https://anonymous.4open.science/r/SpikySpace}{https://anonymous.4open.science/r/SpikySpace}
\end{abstract}

\section{Introduction}
Analytics systems deployed at the edge, such as urban traffic management, industrial monitoring, and on-device sensing, demand forecasting methods that function locally~\cite{lv2022edge, reis2025edge, kashpruk2023time}.
Driven by strict limits on power, privacy, and bandwidth, these applications require predictions in real time without reliance on the cloud.
While accurate, deep learning models rely on dense matrix multiplications that consume substantial energy, making them difficult to deploy on devices with limited power reserves. 
Spiking neural networks (SNNs) offer a promising alternative~\cite{han2022review}. By replacing complex multiplications with discrete additions driven by events, SNNs enable computation to scale sparsely with signal changes~\cite{yamazaki2022spiking}. This mechanism is naturally suited for processing streaming data, where efficiency and responsiveness are essential.


Despite these advantages, recent advancements in SNNs have yet to translate theoretical efficiency into practice. Works like SpikeTCN~\cite{lv2024efficient}, iSpikformer~\cite{lv2024efficient}, and TS-LIF~\cite{shibo2025ts} have achieved desirable performance, while they still rely on high-cost encoders or the transformer architecture. Even with spikes, the underlying mechanism of self-attention exhibits quadratic complexity with respect to the length of the sequence, which causes the costs of memory access and communication to explode over time.
This fundamentally undermines the energy benefits of event-driven computation, highlighting the need for a more efficient architectural alternative.

To bridge this gap, we propose SpikySpace, a novel spiking architecture that leverages State Space Models (SSMs)~\cite{gu2021efficiently, gu2024mamba} to achieve high accuracy under strict energy constraints. SpikySpace adopts an architecture based on event-driven state spaces that combines selective updates driven by spikes with continuous time dynamics.
Unlike recent SNNs constrained by the quadratic complexity of attention mechanisms or the sequential bottlenecks of RNNs, SpikySpace employs compact latent states updated with $O(L)$ complexity. 
This design eliminates the heavy $O(L^2)$ computation and memory overhead of attention mechanisms while capturing long-range dependencies more effectively than traditional recurrence.
Consequently, SpikySpace achieves competitive accuracy while significantly reducing energy consumption, positioning it as an ideal solution for deployment on resource-constrained neuromorphic and edge platforms.

However, implementing SSMs on neuromorphic hardware remains challenging due to the continuous nature of the state selection process and the complexity of activation functions.
Existing efforts to integrate SNNs and SSMs~\cite{shen2025spikingssms, bal2024p, zhong2024spike, huang2025spikingmamba} often compromise efficiency by retaining dense floating point multiplications in the recurrent scan or relying on expensive operations like exponentials.
SpikySpace overcomes these limitations through a fully discrete hardware-friendly design. 
{First}, we introduce a selective scan mechanism driven by spikes.
Unlike standard implementations that update states via continuous matrix multiplications, we constrain the parameters for time scales to powers of two and encode all intermediate activations and recurrent states into sparse spikes. This transforms the core state transitions into efficient bit-shift operations and integer accumulations, free of multiplication.
{Next}, we propose approximations we called PTsoftplus and PTSiLU to handle the complex nonlinearities. They utilize linear transformations and power-of-two components to eliminate costly exponentials and divisions.
We will give theoretical proofs of their error bounds, guaranteeing that the surrogate gradients remain stable during backpropagation. To enable effective learning, we employ a quantization-aware training strategy in the continuous domain, followed by a deterministic conversion to spike events. To the best of our knowledge, SpikySpace is the first spiking state-space model to realize fully multiplication-free recurrent dynamics.

Our evaluations on four multivariate forecasting benchmarks demonstrate that SpikySpace maintains high predictive performance while drastically lowering resource usage.
Specifically, SpikySpace surpasses the accuracy of the previous best-performing SNN SpikeSTAG~\cite{hu2025spikestag} by 3.0\% on dataset Metr-la~\cite{li2018diffusion}, and even outperform dense iTransformer~\cite{liuitransformer} on dataset Electricity~\cite{lai2018modeling}.
SpikySpace also reduces energy consumption by factors of $26.6\times$ compared to iSpikformer and $78.9\times$ compared to iTransformer, without compromising stability.

Our main contributions are summarized as follows:
\begin{itemize}
\item We propose SpikySpace, a novel architecture that integrates the continuous dynamics of SSMs into an event-driven framework for time series tasks. 
By implementing a spike-driven selective scan constrained to powers of two, we achieve a recurrent backbone with linear complexity that is completely free of multiplication.

\item We design and theoretically validate two hardware-friendly approximations, PTsoftplus and PTSiLU. Constructed from linear transformations and power-of-two components, these operators are proven to approximate the original functions closely while eliminating costly exponentials and divisions, removing a major obstacle to deploying SSM-based SNNs on energy-constrained systems.

\item We validated SpikySpace on multivariate forecasting benchmarks, and show that it achieves competitive accuracy with substantial savings in energy and parameter size. We believe this is evidence that SpikySpace is a highly practical solution for edge deployment.
\end{itemize}

\section{Related Work}
\paragraph{Spiking Neural Networks for Sequence Modeling.}
Spiking Neural Networks (SNNs) offer significant energy efficiency through event-driven computation~\cite{taherkhani2020review, zhou2024direct}. 
While successful in tasks ranging from image classification~\cite {hu2023fast, tang2024onespike} to NLP~\cite{zhou2023spikingformer, tang2024sorbet} with SNNs' intrinsic temporal dynamics, modeling long sequences remains a challenge. Current Spiking Transformers capture global dependencies but suffer from the quadratic complexity of attention mechanisms, which undermines the sparse efficiency of neuromorphic computing. 
This bottleneck necessitates exploring linear-complexity backbones, such as State Space Models (SSMs)~\cite{gu2021efficiently}, to better align with the efficiency goals of SNNs.

\paragraph{Spiking State Space Models.}
Adapting the continuous, dense dynamics of SSMs to the discrete, sparse domain of SNNs presents significant challenges.
Although recent studies have explored integrating SNNs with SSMs~\cite{shen2025spikingssms, bal2024p, zhong2024spike, huang2025spikingmamba}, these implementations are predominantly hybrid.
Most retain real-valued state updates, dense matrix multiplications, or complex continuous activations like Softplus and SiLU within the recurrent loop.
Such designs fail to fully exploit the energy benefits of neuromorphic hardware, as the core state evolution remains reliant on floating-point arithmetic.
In contrast, SpikySpace proposes a fully spiking backbone where latent states are updated via sparse bit shift operations and additions.

\paragraph{Time-series Forecasting.}
Forecasting time series is critical for domains ranging from industrial monitoring to energy management. 
While transformers~\cite{informer, autoformer, patchtst} currently forecast with high accuracy, their dense attention maps and KV-cache requirements make them memory-bound and prohibitive for deployment on the edge. Model compression techniques like quantization and pruning~\cite{model_compression_ts}
alleviate some overhead but do not alter the underlying dense computation. 
Consequently, SNNs have been investigated for efficient forecasting~\cite{lv2024efficient, shibo2025ts, hu2025spikestag}. However, existing SNN forecasters often face a trade-off where they either rely on structures similar to RNNs that struggle with long horizons or adopt transformer architectures that suffer from quadratic complexity. There remains a lack of neuromorphic forecasting solutions that are both highly accurate and energy efficiency.

\section{Preliminary}
\paragraph{Problem Statement.}
Let $\mathbf{X} = \{\mathbf{x}_1, \ldots, \mathbf{x}_T\} \in \mathbb{R}^{T \times N}$ denote a multivariate time series with $N$ variables. Given a historical window of length $L$, $\mathbf{X}_{in} \in \mathbb{R}^{L \times N}$, the goal is to forecast the future horizon $H$, denoted as $\mathbf{Y}_{pred} \in \mathbb{R}^{H \times N}$. We aim to learn a mapping $f_\theta: \mathbf{X}_{in} \mapsto \mathbf{Y}_{pred}$ parameterized by an SNN, leveraging event-driven computation to process the temporal stream efficiently.

\paragraph{Spiking Neural Networks.}
SNNs process information through discrete spike events $s_t \in \{0, 1\}$ rather than continuous activations.
To bypass the training difficulties caused by non-differentiable spike generation, we utilize the ANN-to-SNN conversion framework, which maps the weight of a pre-trained ANN to a spiking equivalent.
Our architecture is built upon the Integrate-and-Fire (IF) neuron. In the discrete time domain, the membrane potential $V_t$ simply accumulates synaptic input $I_t$. Upon exceeding a threshold $\theta$, the neuron fires as $s_t = 1$ and resets via subtraction as $V_t \leftarrow V_t - \theta$.
For hardware efficiency, we adopt the Average Integrate-and-Fire (Avg-IF) variant~\cite{yan2025low}, which optimizes memory access patterns by synchronizing potential updates. Detailed algorithms are provided in Appendix~\ref{app:alg}. 
In this work, we denote the spike neuron used to generate spike trains as $\mathcal{SN}(\cdot)$, which accumulates synaptic inputs and triggers a binary output upon exceeding a firing threshold $\theta$.

\paragraph{State Space Models.}
State Space Models (SSMs) capture sequence dynamics through a latent state $h_t \in \mathbb{R}^D$. 
While theoretically grounded in continuous differential equations, efficient inference relies on the discretized recurrence:
\begin{align}
\label{eq:ssm_discrete}
h_t = \bar{\mathbf{A}}h_{t-1} + \bar{\mathbf{B}}x_t, \quad y_t = \mathbf{C}h_t,
\end{align}
where $\bar{\mathbf{A}}$ and $\bar{\mathbf{B}}$ are the discretized system parameters.
We provide the detailed derivation from the continuous differential equations to this discrete form in Appendix~\ref{app:ssm}. As Eq.~\eqref{eq:ssm_discrete} relies on dense floating-point multiplications for both state transitions and input projections. SpikySpace eliminates these bottlenecks through a strategic combination of spike-based encoding and power-of-two quantization.

\section{Method}
\subsection{Overall Framework}
\label{subsec:framework}

\begin{figure*}[t]
  \centering
  \includegraphics[width=0.9\textwidth]{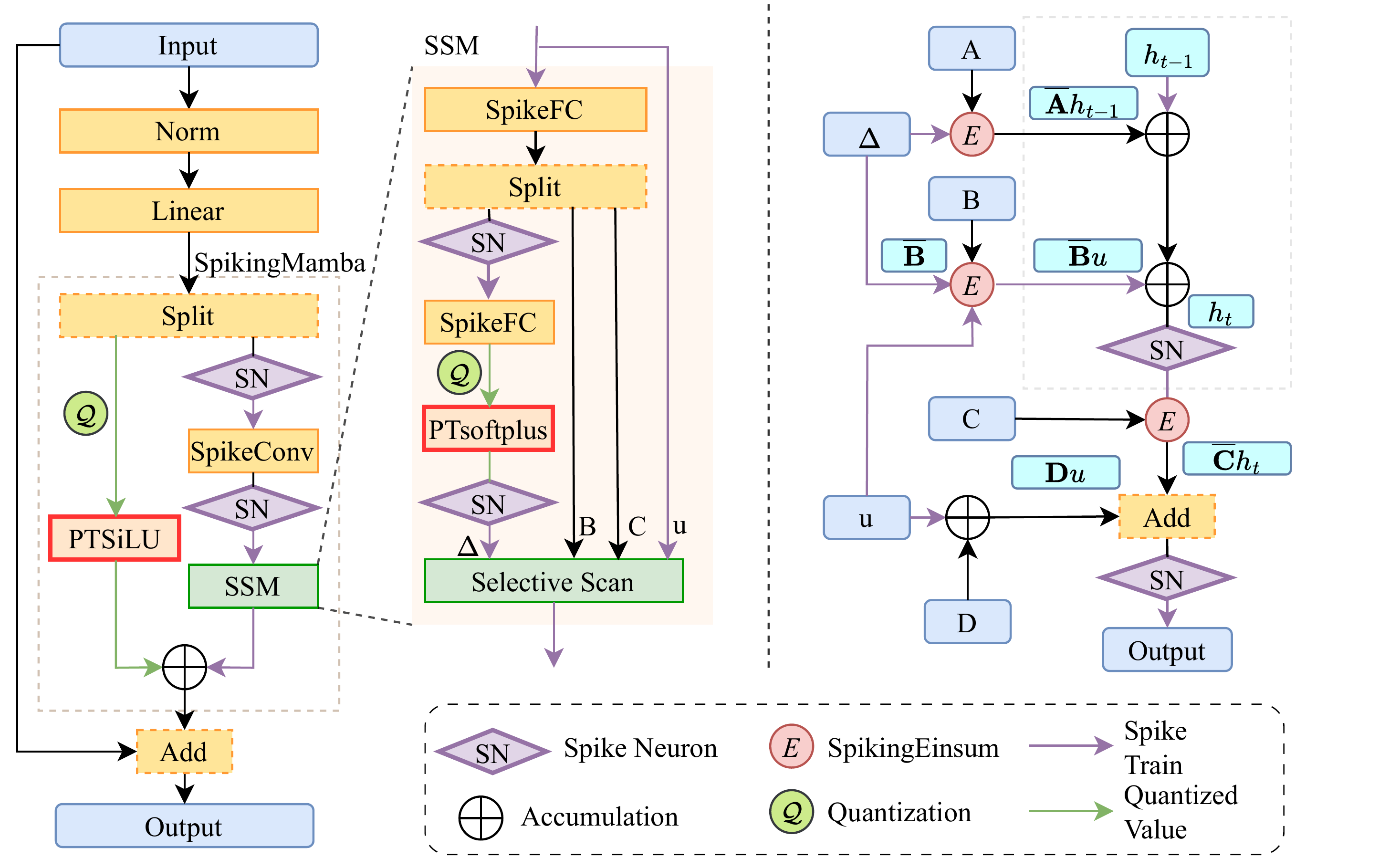}
  \caption{
  The overall architecture of SpikySpace. (Left) The model consists of stacked Spiking Mamba blocks followed by a prediction head. Each block employs a standard pre-norm residual design. Normalization is performed via RMSNorm, which admits hardware-efficient approximations~\cite{tang2024sorbet} for deployment. (Right) Inside the block, input features are projected and split into two branches comprising a residual gating branch and a data branch processed by our Spiking SSM. Weights are quantized to INT8, and activations are encoded as spikes, enabling an inference pipeline free of multiplication.
  } 
  \label{fig:main-figure}
\end{figure*}

As illustrated in Figure~\ref{fig:main-figure}, SpikySpace stacks multiple Spiking Mamba blocks to process multivariate time series. Each block integrates our proposed spiking scan mechanism within a residual backbone to capture temporal dependencies. Finally, a linear head maps the learned representations to the target prediction horizon.

We implement a co-design strategy focused on efficiency. Specifically, SSMs weight are quantized to 8-bit integers via post-training quantization, while dynamic activations are encoded as binary spikes. This combination allows us to replace dense matrix products with sparse accumulations of integer weights of lower bitwidths. 
Furthermore, we replace expensive operations in standard activations such as Softplus and SiLU with approximations based on bit shifts, resulting in a purely additive SSMs inference ideally suited for energy-constrained edge devices.



\subsection{State Space Module}
The Spiking State Space Module serves as the core component of SpikySpace, adapting the continuous dynamics of Mamba into a sparse mechanism driven by events. Functionally, we organize this module into sequential stages comprising convolutional spike encoding, selective parameterization, the spiking selective scan, and output gating.

\paragraph{Local Feature Encoding.}
Given the input data path $\mathbf{x}_{\text{in}} \in \mathbb{R}^{B \times L \times D}$, we first encode the continuous features into an initial spike train $\mathbf{s}_{\text{in}}$ via the spike neuron $\mathcal{SN}$.
Subsequently, to capture local temporal contexts, $\mathbf{s}_{\text{in}}$ undergoes a depth-wise separable 1D convolution followed by another spike neuron.
The encoding process is formalized as:

\begin{equation}
\mathbf{s_{\text{in}}} = \mathcal{SN}(\mathbf{x}_{\text{in}}),
\end{equation}
\begin{equation}
\mathbf{s} = \mathcal{SN}(\text{Conv1D}(\mathbf{s}_{\text{in}}),
\end{equation}
Here, $\mathbf{s} \in \{0, 1\}^{B \times L \times D}$ serves as the sparse input to the subsequent SSM.

\paragraph{Selective Parameterization.}
Given the input spike train $\mathbf{s}$, we instantiate a latent dynamical system governed by the parameters $(\boldsymbol{\Delta}, \mathbf{A}, \mathbf{B}, \mathbf{C}, \mathbf{D})$. Specifically, $\boldsymbol{\Delta}$ represents the discretization step size, while $(\mathbf{A}, \mathbf{B}, \mathbf{C}, \mathbf{D})$ define the state evolution and linear projections.

Let $N$ be the SSM state dimension and $r_{\Delta}$ be the rank of the step-size projection. We first apply a linear transformation with weights $\mathbf{W} \in \mathbb{R}^{d \times (r_{\Delta} + 2N)}$ to the input $\mathbf{s}$ and split the result into the projection weights $\mathbf{B}, \mathbf{C} \in \mathbb{R}^{B \times L \times N}$ and the raw step size $\boldsymbol{\Delta}_\text{raw} \in \mathbb{R}^{B \times L \times r_{\Delta}}$. Then $\boldsymbol{\Delta}_\text{raw}$ is passed through a spiking neuron to generate intermediate spikes, denoted as $\boldsymbol{\Delta}_\text{mid}$. These spikes are then processed via a quantized linear transformation and fed into our proposed PTSoftplus.
Finally, a second spiking neuron encodes the output into the binary gating signal $\boldsymbol{\Delta}_s$.
The entire parameter generation process is formulated as:
\begin{equation}
\begin{aligned}
\boldsymbol{\Delta}_\text{raw}, \mathbf{B}, \mathbf{C} & = \text{Split}(\mathbf{W}\mathbf{s}  + \mathbf{b}), \quad \boldsymbol{\Delta}_\text{mid} = \mathcal{SN}\!\left(\boldsymbol{\Delta}_\text{raw}\right), \\
\tilde{\boldsymbol{\Delta}} &= \mathcal{Q}(\mathbf{W}_\Delta \boldsymbol{\Delta}_\text{mid} + \mathbf{b}_\Delta),\\
\boldsymbol{\Delta}_\text{out} &=\mathrm{PTSoftplus}\!\left(\tilde{\boldsymbol{\Delta}}\right), \quad \boldsymbol{\Delta}_s = \mathcal{SN}\!\left(\boldsymbol{\Delta}_\text{out}\right),  \\
\end{aligned}
\end{equation}
Here, $\mathrm{PTSoftplus}$ is our hardware-friendly approximation of $\mathrm{Softplus}$ introduced in Section \ref{PTSoftplus}, and $\mathcal{Q}$ represents the quantization process as Equation~\ref{eq: cq}.
Together with the learnable parameters $\mathbf{A}$ and $\mathbf{D}$, these inputs are forwarded to the selective scan module.

\paragraph{Spiking Selective Scan.}
With the \textbf{spiking step size $\boldsymbol{\Delta}_s$} and discrete parameters, the system enters the recursive scan stage.
In standard Mamba, the transition matrix is obtained through an exponential operation $\overline{\mathbf{A}}_t = \exp(\boldsymbol{\Delta}_t \mathbf{A})$, which requires computationally intensive floating-point exponentiation.
In SpikySpace, the discretization is simplified by the binary nature of the step size $\boldsymbol{\Delta}_s$. 
Instead of continuous decay, the state update becomes a conditional bit-shift operation applied to the hidden state.
Specifically, we pre-compute the shift amount $\mathbf{K} = \lfloor \mathbf{A} \rceil$ derived from the static dynamics parameter.
At each time step, if the spike $\boldsymbol{\Delta}_s$ is active as a $1$, the state undergoes a bit-shift; otherwise, it remains preserved.
This effectively replaces the computationally expensive exponential operation with hardware-intrinsic bit operations:
\begin{equation}
\label{eq:shift_logic}
\overline{\mathbf{A}}_t =
\begin{cases}
    2^{\mathbf{K}} 
    & \text{if } \boldsymbol{\Delta}_{s,t} = 1, \\
    \mathbf{I} & \text{if } \boldsymbol{\Delta}_{s,t} = 0.
\end{cases}
\end{equation}
Importantly, replacing the natural exponential with a base-2 formulation does not reduce the expressive power of the model. Since $2^x = e^{\ln(2)x}$, the above transition is equivalent to scaling the effective step size by a constant factor $\ln (2)$. This corresponds to a fixed rescaling of the temporal dynamics, which can be fully compensated by the learnable parameters during training.

To maintain sparsity, we enforce the hidden state $\mathbf{h}_t$ to be event-driven. The recurrent dynamics at timestep $t$ are defined as:
\begin{align}
\overline{\mathbf{B}}_t &= \boldsymbol{\Delta}_{s,t} \mathbf{B}_t, \\
\mathbf{h}_t &= \mathcal{SN}(\overline{\mathbf{A}}_t \mathbf{h}_{t-1} + \overline{\mathbf{B}}_t \mathbf{s}_t),
\end{align}
Here, $\mathbf{h}_t \in \{0, 1\}^{B\times d_{\text{hidden}}\times n \times T}$ denotes the binary hidden state. 
Since $\mathbf{s}_t$ is binary, the input term $\overline{\mathbf{B}}_t \mathbf{s}_t$ reduces to sparse accumulation.

Finally, the module output $\mathbf{y}_t$ is computed via the dynamic projection $\mathbf{C}_t$ and the skip connection $\mathbf{D}$:
\begin{equation}
\mathbf{y}_t = \mathcal{SN}(\mathbf{C}_t \mathbf{h}_t + \mathbf{D}\mathbf{s}_t).
\end{equation}
This entire scanning mechanism achieves $O(L)$ complexity while operating strictly with accumulations and bit-shifts, making it highly efficient for edge deployment. The detailed procedure is summarized in Algorithm \ref{alg:sqssm}.

\begin{algorithm}[t]
\caption{SpikingMamba with Selective Scan}
\label{alg:sqssm}
\begin{algorithmic}[1]
\STATE \textbf{Input:} Spike train $\mathbf{s} \in \{0,1\}^{B\times L\times d}$; sequence length $L$
\STATE \textbf{Require:} Linear Projection Weight $\mathbf{W}$, $\mathbf{W}_{\Delta}$, Bias $\mathbf{b}$, $\mathbf{b}_{\Delta}$, SSM dynamics $\mathbf{A} \leftarrow -\exp(\mathbf{A}_{\log})$,\;
          $\mathbf{D}$
\STATE \textbf{Output:} Predicted Spike Train $\mathbf{y} = \{\mathbf{y}_t\}_{t=1}^{L}$

\STATE \textbf{Stage 1: Parameters Preparation} 
\STATE $\mathbf{h}_0 \leftarrow \mathbf{0}$ \COMMENT{Initialize the Hidden State $\in \{0,1\}$}

\STATE $\mathbf{U}_{\text{proj}} \leftarrow \text{IntegerAdd}(\mathbf{W}\mathbf{s}, \mathbf{b})$ \COMMENT{Sparse addition triggered by input spikes}
\STATE $(\boldsymbol{\Delta}_\text{raw}, \mathbf{B}, \mathbf{C}) \leftarrow \mathrm{Split}(\mathbf{U}_{\text{proj}})$
\STATE $\boldsymbol{\Delta}_\text{mid} \leftarrow \mathcal{SN}(\boldsymbol{\Delta}_\text{raw})$ \COMMENT{Encode $\boldsymbol{\Delta}_\text{mid}$ into spikes}
\STATE  $\tilde{\boldsymbol{\Delta}} \leftarrow \mathcal{Q}(\mathbf{W}_{\Delta}\boldsymbol{\Delta}_\text{mid} + \mathbf{b}_{\Delta})$
\STATE $\boldsymbol{\Delta}_s \leftarrow \mathcal{SN}(\mathrm{PTSoftplus}(\tilde{\boldsymbol{\Delta}}))$ \COMMENT{$\tilde{\boldsymbol{\Delta}}$ is integer}

\STATE \textbf{Stage 2: Selective Scan} 
\STATE {\texttt{// Keep $\boldsymbol{\Delta}_s, h_t, s_t$ as spikes}}
\FOR{$t = 1$ \TO $L$}
    \STATE $\overline{\mathbf{A}}_t =
\begin{cases}
    2^{\mathbf{K}}, \mathbf{K}=\lfloor A \rceil
    & \text{if } \boldsymbol{\Delta}_{s,t} = 1, \\
    \mathbf{I} & \text{if } \boldsymbol{\Delta}_{s,t} = 0.
\end{cases}$
    \STATE $\overline{\mathbf{B}}_t \leftarrow 
\begin{cases}
    \mathbf{B}_t
    & \text{if } \boldsymbol{\Delta}_{s,t} = 1, \\
    0 & \text{if } \boldsymbol{\Delta}_{s,t} = 0.
\end{cases}$
    \STATE $\mathbf{h}_t \leftarrow \mathcal{SN}(\overline{\mathbf{A}}_t \mathbf{h}_{t-1} + \overline{\mathbf{B}}_t \mathbf{s}_t)$
    \STATE $\mathbf{y}_t \leftarrow \mathcal{SN}(\mathbf{C}_t \mathbf{h}_t + \mathbf{D}\mathbf{s}_t)$ \COMMENT{Encode Output into Spikes}
\ENDFOR
\vspace{0.3em}
\STATE \textbf{return} $\mathbf{y}$
\end{algorithmic}
\end{algorithm}

\paragraph{Output Gating.}
The final stage employs a residual branch to stabilize the learned representations.
The residual signal $\mathbf{x}_{\text{res}}$ is processed through quantization and our PTSiLU activation, before being modulated by the scan output $\mathbf{y}$, followed by an output linear projection:
\begin{equation}
\mathbf{\tilde y} = \mathbf{y} \odot \text{PTSiLU}(\mathcal{Q}\!\left(\mathbf{x}_{\text{res}}\right))
\end{equation}
\begin{equation}
\mathbf{z} = \mathbf{W}_{\text{o}} \mathbf{\tilde y} + \mathbf{b}_{\text{o}}.
\end{equation}

Here, the binary spike train $\mathbf{y}$ acts as a sparse mask.
As introduced in Section~\ref{PTSiLU}, PTSiLU employs power-of-two scaling, the non-zero elements in $\tilde{\mathbf{y}}$ remain compatible with bit-wise logic.
This transforms the final projection with the weight of $\mathbf{W}{\text{o}}$ into a sparse shift-and-accumulate operation, allowing the entire block to bypass dense multiplication.

\subsection{Power-of-Two Softplus}
\label{PTSoftplus}
Standard SSMs rely on the Softplus function, defined as $\mathrm{Softplus}(x) = \ln(1 + e^x)$, to ensure the positivity of the timescale parameter $\boldsymbol{\Delta}$. As Softplus depends heavily on operations like exponential and logarithmic that are computationally prohibitive for neuromorphic hardware, we propose Power-of-Two Softplus(PTSoftplus) to resolve this. By utilizing exclusively power-of-two scaling, PTSoftplus replaces expensive floating-point units with efficient bit-shifts and additions, while preserving the convexity required for stable dynamics. 
Specifically, it is defined as the following function:
\begin{equation}
\\ 
\text{PTSoftplus}(x) =
\begin{cases}
2^x, & \text{if } x < x_c, \\
x + C, & \text{otherwise},
\end{cases}
\end{equation}
To guarantee $C^1$ continuity at the junction $x_c$, the constants are analytically derived as:
\[
\begin{gathered}
    x_c = \log_2 \big({1 \over \ln 2}\big) \approx 0.5288, \\
    C = {1\over \ln 2} - x_c \approx 0.9139.
\end{gathered}
\]
\paragraph{Theoretical Properties.} We provide theoretical guarantees to ensure that our approximation supports stable gradient-based training.

\begin{lem}\label{lem:ptsoftplus_continuous}
    {\rm PTSoftplus} is continuously differentiable.
\end{lem}

The proof of Lemma~\ref{lem:ptsoftplus_continuous} is deferred to Appendix~\ref{sec:ptsoftplus_continuous}.
Next, we demonstrate that the proposed PTSoftplus function closely approximates Softplus. Specifically, the functional deviation is bounded by a small constant. The derivatives of PTSoftplus and Softplus also exhibit pointwise closeness, ensuring the approximation does not detrimentally affect the model's training dynamics.

\begin{lem}\label{lem:softplus_closeness}
    The maximum deviation between the {\rm PTSoftplus} and the {\rm Softplus} function is bounded by 0.914, i.e.
    \[
        \|{\rm PTSoftplus} - {\rm Softplus}\|_\infty  \le 0.914.
    \]
    Moreover, the maximum deviation between the derivatives of {\rm PTSoftplus} and {\rm Softplus} is bounded by 0.371, i.e. 
    \[
        \|{\rm PTSoftplus}' - {\rm Softplus}'\|_\infty \le 0.371.
    \]
\end{lem}

\begin{figure}
\begin{subfigure}{.4\linewidth}
    \centering
    \includegraphics[width=\linewidth]{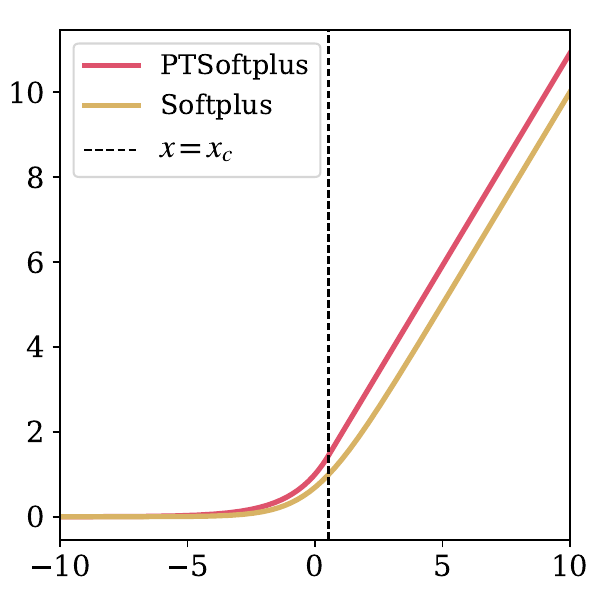}
\end{subfigure}
\begin{subfigure}{.6\linewidth}
    \centering
    \includegraphics[width=\linewidth]{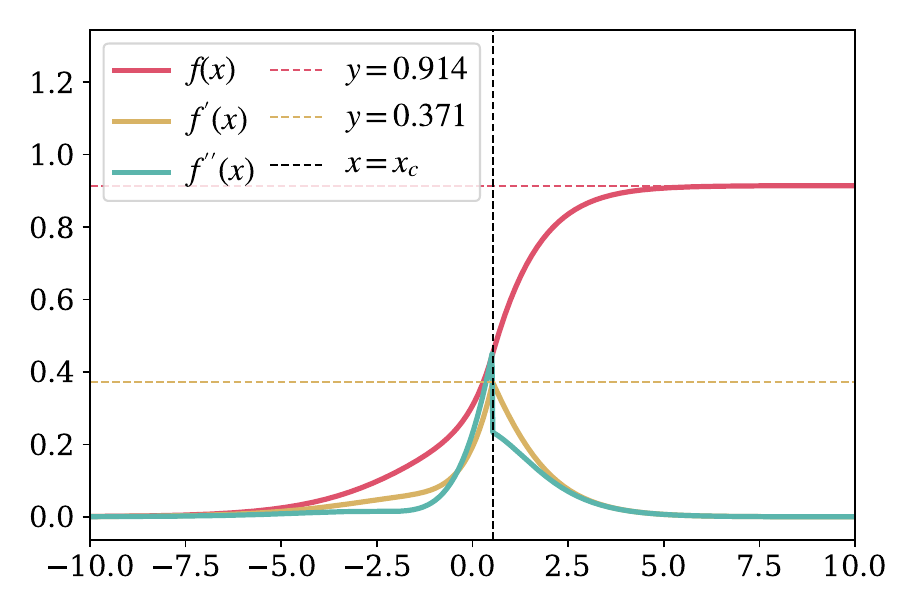}
\end{subfigure}
    \caption{{\em Left:} The PTSoftplus and the Softplus function. {\em Right:} The bounds on the deviations between these two functions.}
    \label{fig:softplus}
\end{figure}

The detailed proof of Lemma~\ref{lem:softplus_closeness} is relegated to Appendix~\ref{sec:softplus_closeness}. Fig.~\ref{fig:softplus} visually demonstrates the function shapes and the deviations between the two functions and their respective derivatives.

\subsection{Power-of-Two SiLU}
\label{PTSiLU}
Similarly, the output gating branch relies on the SiLU activation, defined as $\mathrm{SiLU}(x) = x \cdot \sigma(x)$, sharing similar hardware bottlenecks as Softplus.
Thus, we propose Power-of-Two SiLU (PTSiLU), a piecewise approximation tailored for hardware consistency defined as the following function:



\begin{equation}
\label{eq:PTswish_def}
\text{PTSiLU}(x) = 
\begin{cases}
-2^x, & \text{if } x < \bar{x}_c, \\
2^{-x - 1} + x + \bar{C}, & \text{otherwise},
\end{cases}
\end{equation}
where the parameters
\[
\begin{gathered}
    \bar{x}_c =  \log_2\Big( {\sqrt{1 + 2 \cdot (\ln 2)^2} - 1 \over 2 \ln 2} \Big) \approx -1.7920, \\
    \bar{C} = -\frac{\sqrt{1 + 2 \cdot (\ln 2)^2}}{\ln 2} - \bar{x}_c \approx -0.2282.
\end{gathered}
\]
\paragraph{Theoretical Properties.} Similar to PTSoftplus, we first provide theoretical guarantees to ensure that PTSiLU supports stable gradient-based training.

\begin{lem}\label{lem:ptsilu_continuity}
    {\rm PTSiLU} is continuously differentiable. 
\end{lem}

The proof of Lemma~\ref{lem:ptsilu_continuity} is deferred to Appendix~\ref{sec:ptsilu_continuity}.
Next, we show that PTSiLU closely approximates SiLU, evidenced by the bounded difference between their function values. Furthermore, the proximity of their first-order derivatives ensures that this approximation does not impair the convergence or stability of the model's training process.

\begin{lem}\label{lem:ptsilu_closeness}
    The maximum deviation between the {\rm PTSiLU} and the {\rm SiLU} function is bounded by 0.316, i.e.
    \[
        \|{\rm PTSiLU} - {\rm SiLU} \|_\infty \le 0.316.
    \]
    Moreover, the maximum deviation between the derivatives of {\rm PTSiLU} and {\rm SiLU} is bounded by $0.263$, i.e.
    \[
        \|{\rm PTSiLU}' - {\rm SiLU}' \|_\infty \le 0.263.
    \]
\end{lem}

\begin{figure}
\begin{subfigure}{.4\linewidth}
    \centering
    \includegraphics[width=\linewidth]{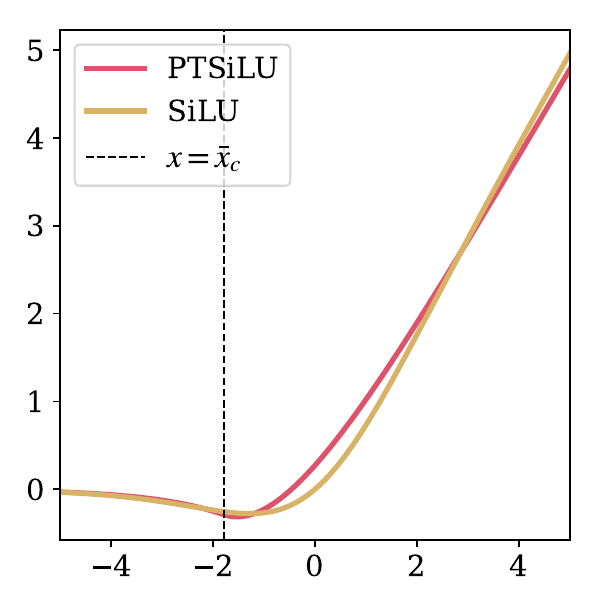}
\end{subfigure}
\begin{subfigure}{.6\linewidth}
    \centering
    \includegraphics[width=\linewidth]{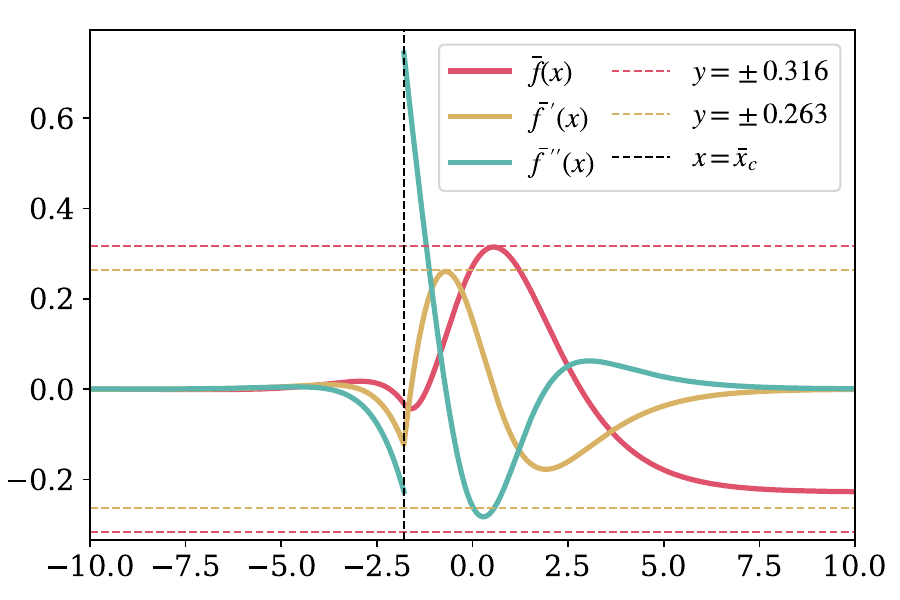}
\end{subfigure}
    \caption{{\em Left:} The PTSiLU and the SiLU function. {\em Right:} The bounds on the deviations between these two functions.}
    \label{fig:silu}
\end{figure}

The proof of Lemma~\ref{lem:ptsilu_closeness} is in Appendix~\ref{sec:ptsilu_closeness}. Fig.~\ref{fig:silu} visually confirms the small deviations in both the function shapes and the derivatives of the two functions.

\subsection{Training Process}
As we adopted an ANN-to-SNN conversion method to obtain our SpikySpace, we first trained a quantized ANN model with PTsoftplus and PTSiLU as the activation functions and then converted it into an SNN model.
For activation quantization, we follow BiT~\cite{liu2022bit} and Sorbet~\cite{tang2024sorbet} to employ Learned Step-size Quantization (LSQ). The quantized activation $X_Q$ is derived from the real-valued input $X_R$ as:

\begin{equation}
\label{eq: cq}
\mathcal{Q}(X_R) = \alpha \cdot 
\operatorname{clip}\!\left(
\operatorname{round}\!\left(\frac{X_R - \beta}{\alpha}\right),
Q_n,\; Q_p
\right)
+ \beta
\end{equation}

where $\alpha$ is the learnable step size and $\beta$ is the offset which can be further quantized to power-of-two values without accuracy loss as discussed in~\cite{tang2024sorbet}. Detailed settings are provided in Appendix~\ref{app:taq}.

\begin{table*}[t]
\centering
\caption{Experimental results of time-series forecasting across four datasets. Models are listed as rows. For each dataset, results include horizons $L \in \{3, 6, 12, 24\}$. The best and second-best results in \textbf{SNNs} are highlighted in \textbf{bold} and \underline{underlined}. `$\uparrow$' means `the higher the better', and `$\downarrow$' means the opposite, i.e., `the lower the better'.}
\label{tab:forecast_comparison}

\setlength{\tabcolsep}{3pt}      
\renewcommand{\arraystretch}{1.05}  

\resizebox{\textwidth}{!}{
\begin{tabular}{l c c  cccc cccc cccc cccc c}
\toprule
\multirow{2}{*}{Model} & \multirow{2}{*}{SNN} & \multirow{2}{*}{Metric} &
\multicolumn{4}{c}{\textbf{Metr-la}} & 
\multicolumn{4}{c}{\textbf{Pems-bay}} & 
\multicolumn{4}{c}{\textbf{Solar}} & 
\multicolumn{4}{c}{\textbf{Electricity}} & 
\multirow{2}{*}{\textbf{Avg.}} \\
\cmidrule(lr){4-7} \cmidrule(lr){8-11} \cmidrule(lr){12-15} \cmidrule(lr){16-19}
& & & 3 & 6 & 12 & 24 & 3 & 6 & 12 & 24 & 3 & 6 & 12 & 24 & 3 & 6 & 12 & 24 \\
\midrule

\multirow{2}{*}{GRU} & \multirow{2}{*}{$\times$} 
& $R^2$$\uparrow$   & .803 & .761 & .682 & .614
& .783 & .769 & .696 & .696 
& .962 & .950 & .907 & .875
& .983 & .981 & .980 & .972 
& .834 \\
& & RRSE$\downarrow$ & .448 & .507 & .585 & .663
& .479 & .504 & .638 & .638
& .508 & .548 & .569 & .572
& .518 & .522 & .531 & .506 
& .546 \\
\midrule

\multirow{2}{*}{iTransformer} & \multirow{2}{*}{$\times$}
& $R^2$$\uparrow$   & .864 & .849 & .763 & .538
 & .938 & .888 & .797 & .629
 & .974 & .964 & .918 & .879
 & .983 & .977 & .977 & .977
 & .870 \\
&& RRSE$\downarrow$ & .344 & .410 & .514 & .652
 & .269 & .362 & .488 & .659
 & .562 & .584 & .575 & .541
 & .213 & .506 & .460 & .305 
 & .465 \\
\midrule

\multirow{2}{*}{Mamba} & \multirow{2}{*}{$\times$}
& $R^2$$\uparrow$   & .891 & .842 & .754 & .612
& .922 & .875 & .778 & .606
&.973 & .954 & .915 & .824
& .989 & .987 & .986 & .984 
& .868 \\
&& RRSE$\downarrow$  & .349 & .419 & .524 & .658
& .303 & .383 & .509 & .679
& .168 & .220 & .300 & .430
& .183 & .203 & .213 & .225 
& .360 \\
\midrule
\midrule

\multirow{2}{*}{SpikeTCN} & \multirow{2}{*}{$\checkmark $}
& $R^2$$\uparrow$   & .845 & .799 & .718 & .602
 & .862 & .829 & .782 & \underbar{.681}
 & .946 & .937 & .893 & .840
 & .974 & .970 & .968 & .963 
 & .851 \\
&& RRSE$\downarrow$ & .415 & .473 & .560 & .665
 & .401 & .448 & \underbar{.504} & \textbf{.582}
 & \underbar{.205} & .252 & .409 & .541
 & .324 & .333 & .338 & .342
 & .425 \\
\midrule

\multirow{2}{*}{SpikeRNN} & \multirow{2}{*}{$\checkmark $}
& $R^2$$\uparrow$   & .784 & .731 & .661 & .557
 & .763 & .721 & .710 & \textbf{.693}
 & .933 & .923 & .903 & .820
 & .984 & .978 & .979 & .964 
 & .819 \\
&& RRSE$\downarrow$ & .490 & .547 & .614 & .702 
 & .527 & .571 & .582 & \underbar{.599}
 & .246 & .278 & .343 & .425
 & .207 & .280 & .314 & .317 
 & .440 \\
\midrule

\multirow{2}{*}{iSpikformer} & \multirow{2}{*}{$\checkmark $}
& $R^2$$\uparrow$   & .805 & .765 & .723 & .549
 & \textbf{.935} & \textbf{.884} & \textbf{.787} & .622
 & \underbar{.972} & \textbf{.955} & \underbar{.918} & \underbar{.869}
 & .982 & .974 & .973 & .974 
 & .856 \\
&& RRSE$\downarrow$ & .466 & .512 & .555 & .709 
 & \textbf{.276} & \textbf{.369} & \textbf{.499} & .665
 & .217 & \textbf{.218} & \textbf{.295} & \textbf{.372}
 & .214 & .284 & .284 & .284 
 & .389 \\
\midrule

\multirow{2}{*}{SpikeSTAG} & \multirow{2}{*}{$\checkmark $}
& $R^2$$\uparrow$   & \underbar{.873} & \underbar{.822} & \underbar{.734} & \underbar{.590}
 & .874 & .835 & \underbar{.787} & .626
 & \textbf{.973} & .950 & \textbf{.926} & \textbf{.879}
 & \underbar{.987} & \underbar{.986} & \underbar{.985} & \underbar{.984} 
 & \underbar{.863} \\
&& RRSE$\downarrow$ & \underbar{.375} & \underbar{.430} & \underbar{.535} & \underbar{.685}
& .384 & .439 & .537 & .661
 & .246 & .272 & .315 & \underbar{.390}
 & \underbar{.207} & \underbar{.222} & \underbar{.224}& \underbar{.225}
 & \underbar{.384} \\
\midrule
\midrule
\multirow{2}{*}{\textbf{SpikySpace}} & \multirow{2}{*}{$\checkmark $}
& $R^2$$\uparrow$   & \textbf{.895} & \textbf{.847} & \textbf{.760} & \textbf{.610}
 & \underbar{.921} & \underbar{.866} & .761 & .605
& .971 & \underbar{.953} & .914 & .825 
 & \textbf{.994} & \textbf{.992} & \textbf{.990} & \textbf{.990}  
 & \textbf{.869} \\

&& RRSE$\downarrow$ & \textbf{.342} & \textbf{.413} & \textbf{.517} & \textbf{.659}
 & \underbar{.304} & \underbar{.395} & .528 & .679 
 & \textbf{.176} & \underbar{.223} & \underbar{.301} & .430 
 & \textbf{.137} & \textbf{.157} & \textbf{.176} & \textbf{.179} 
 & \textbf{.351} \\

\bottomrule
\bottomrule
\end{tabular}
}
\label{tab:flat_comparison_with_avg}
\end{table*}

\section{Result}
We evaluate SpikySpace on four multivariate time-series benchmarks across traffic, solar energy, and electricity domains.
The comparison includes mainstream ANNs such as GRU and iTransformer, alongside state-of-the-art SNNs like iSpikformer and SpikeSTAG.
We report the Coefficient of Determination $R^2$ and Root Relative Squared Error (RRSE) across prediction horizons of \{3, 6, 12, 24\} steps following previous works.
All experiments utilize an NVIDIA A100 GPU with a timestep of $T=3$.
Dataset specifications and hyperparameter details are provided in Appendix~\ref{app:exp_details}.

\subsection{Comparing with the Baseline}
As shown in Table~\ref{tab:forecast_comparison}, SpikySpace achieves state-of-the-art SNN performance, ranking first on two datasets and remaining highly competitive on the others.
On the Electricity dataset, our model achieves an average $R^2$ of 0.992, even outperforming dense iTransformer at 0.978.
This result indicates that our method effectively preserves high-precision dynamics despite the quantization.
On traffic dataset Metr-la, SpikySpace surpasses the strongest SNN baseline by a margin of 3.0\% on average (0.778 vs 0.755).
This confirms that the spiking selective scan can capture long-range dependencies better than traditional recurrence.
These results demonstrate that SpikySpace closes the accuracy gap between SNNs and dense ANNs while maintaining neuromorphic efficiency.

\subsection{Energy Analysis}

\begin{table*}[t]
\centering
\caption{Comparison of the model size, energy consumption, and performance across SpikySpace and baselines. The energy reduction is calculated relative to that of iTransformer. The measurements are based on the Electricity dataset with a prediction horizon of 3.}
\label{tab:energy-comparison}
\begingroup
\normalsize
\begin{tabular}{lccccccc}
\toprule
\textbf{Model} & \textbf{Backbone} & \textbf{SNN} & \textbf{Param (M)}  & \textbf{Energy (mJ)} & \textbf{Energy Reduction} & $\boldsymbol{R^2}$$\uparrow$ & $\boldsymbol{RRSE}$$\downarrow$ \\
\midrule
SpikySpace     & SSM               &  \checkmark 
& 0.868 & 0.17 & 98.20\%           & 0.994 & 0.137  \\
SpikeSTAG      & LSTM, transformer &  \checkmark 
& 1.566 & 4.39 & 53.64\%           & 0.987 & 0.207  \\
iSpikformer    &       transformer &  \checkmark 
& 1.634 & 3.19 & 66.30\%           & 0.982 & 0.214  \\
iTransformer   &       transformer &  $\times$     
& 1.634 & 9.47 & /                 & 0.983 & 0.213  \\
\bottomrule
\end{tabular}
\endgroup
\end{table*}

The computational cost of SNNs depends on the time-window length $T$ and the spike rate $s$~\cite{yan2024reconsidering}.
Unlike prior SNN works that estimate energy solely based on synaptic operations, our evaluation explicitly accounts for the hardware bottlenecks, such as memory access and state updates. Specifically, the energy of a spiking layer can be represented as:
\begin{equation}
E_{\text{total}} = T \cdot s \cdot E_{\text{ACC}} + T \cdot s \cdot (E_{\text{move}} + E_{\text{weight}})
\end{equation}

, where $E_\text{ACC}$ is the accumulation energy, $E_\text{move}$ is the spike movement energy and $E_\text{weight}$ is the SRAM weight access energy. All the results are based on the energy data we measured on a commercial 22nm process.

However, in our SpikySpace, the recursive accumulation of the state transition $\overline{\mathbf{A}}_t$ introduces a positive value and tends to saturate the membrane potential.
Without intervention, this drives the firing rate toward $\sim\!100\%$, where spikes occur at every timestep, negating the sparsity benefit.
To address this, we implement a {\em temporal rescaling} strategy. The firing threshold is scaled $T$ times with weights adjusted correspondingly to be $T$ times larger.
This effectively aggregates the information of $T$ steps, reducing the firing frequency from $1$ to $1/T$ while maintaining the consistency of the output signal magnitude.
Empirically, this strategy drastically reduces the average spike rate from $66.0\%$ to $26.8\%$.

As shown in Table~\ref{tab:energy-comparison},  
With a model size that is only  53.1\% to 55.4\% of the baselines, SpikySpace reduced estimated energy consumption by 94.7\%-98.2\%, consuming only 0.17 mJ on the Electricity dataset.

\subsection{Ablation Study}
We will now study the effectiveness of our activation approximations and the model's sensitivity to timestep variations.

\begin{table}[h]
\centering
\caption{Ablation study results of PTSoftplus and PTSiLU on full-precision ANNs.}
\label{tab:ablation_compact}
\resizebox{\columnwidth}{!}{%
\begin{tabular}{l|c|ccccc}
\toprule
\multirow{2}{*}{\textbf{Model Variant}} & \multirow{2}{*}{\textbf{Metric}} & \multicolumn{5}{c}{\textbf{Prediction Horizon}} \\
 & & \textbf{3} & \textbf{6} & \textbf{12} & \textbf{24} & \textbf{Avg.} \\
\midrule

\multirow{2}{*}{SSM-softplus-SiLU} 
 & $R^2 \uparrow$   & \textbf{.895} & .846 & \textbf{.763} & .617 & \textbf{.780} \\
 & RRSE $\downarrow$ & \textbf{.342} & .414 & \textbf{.514} & .653 & \textbf{.481} \\
\midrule

\multirow{2}{*}{SSM-PTsoftplus-SiLU} 
 & $R^2 \uparrow$   & .892 & .845 & .760 & \textbf{.619} & .779 \\
 & RRSE $\downarrow$ & .347 & .416 & .517 & \textbf{.651} & .483 \\
\midrule

\multirow{2}{*}{SSM-softplus-PTSiLU} 
 & $R^2 \uparrow$   & \textbf{.895} & \textbf{.848} & .752 & .613 & .777 \\
 & RRSE $\downarrow$ & \textbf{.342} & \textbf{.412} & .525 & .656 & .484 \\
\midrule

\multirow{2}{*}{SSM-PTsoftmax-PTSiLU} 
 & $R^2 \uparrow$   & .893 & .847 & .755 & .614 & .777 \\
 & RRSE $\downarrow$ & .346 & .413 & .522 & .656 & .484 \\

\bottomrule
\end{tabular}%
}
\end{table}

\begin{table}[]
\centering
\caption{Ablation study on the impact of PTsoftplus and PTSiLU. `Bits' in this table denotes the activation quantization bit-width.}
\normalsize 
\renewcommand{\arraystretch}{0.75} 

\begin{tabular}{c|c|c|c|c}
\toprule
\textbf{Bits} & PTsoftplus  & PTSiLU  & $R^2$ & $\Delta R^2$ \\ 
\midrule
\multirow{4}{*}{\textbf{4}} 
& $\times$      & $\times$      & .607 & N/A   \\
& $\checkmark$  & $\times$      & .611 & +.004 \\
& $\times$      & $\checkmark$  & .611 & +.004 \\
& $\checkmark$  & $\checkmark$  & .610 & +.003 \\
\midrule
\multirow{4}{*}{\textbf{1}} 
& $\times$      & $\times$      & .612 & N/A   \\
& $\checkmark$  & $\times$      & .605 & -.007 \\
& $\times$      & $\checkmark$  & .616 & +.004 \\
& $\checkmark$  & $\checkmark$  & .603 & -.009 \\
\bottomrule
\end{tabular}
\label{abl2}
\end{table}

\textbf{Approximated Operations.}
Table~\ref{tab:ablation_compact} shows that replacing standard Softplus and SiLU with PTSoftplus and PTSiLU causes negligible performance drops across all horizons, confirming the effectiveness of our activation approximation.
We further quantify the impact under different activation bit-widths and demonstrate that our variants remain stable even under extremely low-bit settings, making them highly suitable for efficient hardware implementation. The results are shown in Table~\ref{abl2}.

\textbf{Timestep.}
We then analyze the influence of the simulation timestep $T$.
As shown in Figure~\ref{fig:ablation-timesteps}, the performance is stable as $T$ varies from 1 to 15.
The curves for both Solar and Metr-la datasets are nearly flat, indicating that the model is robust to different update granularities.
We adopt $T=3$ as the default setting to balance computational cost and stability at long prediction horizons.

\begin{figure}[t]
  \centering
  \subfloat[Solar - $R^2$ $\uparrow$]{%
    \includegraphics[width=0.49\linewidth]{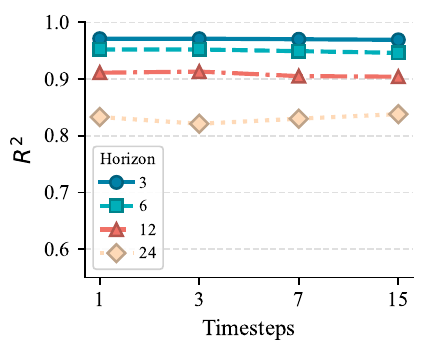}%
  }%
  \hfill
  \subfloat[Metr-la - $R^2$ $\uparrow$]{%
    \includegraphics[width=0.49\linewidth]{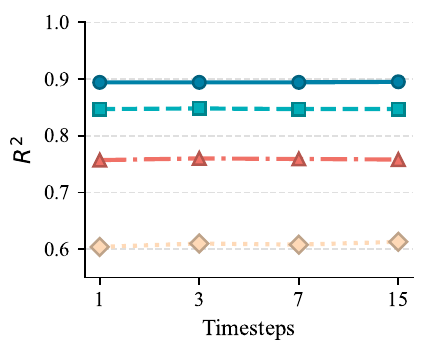}%
  }\\
  \subfloat[Solar - RRSE $\downarrow$]{%
    \includegraphics[width=0.49\linewidth]{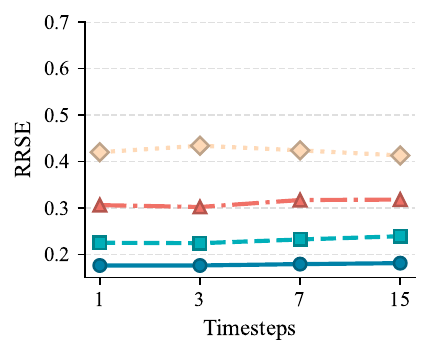}%
  }%
  \hfill
  \subfloat[Metr-la - RRSE $\downarrow$]{%
    \includegraphics[width=0.49\linewidth]{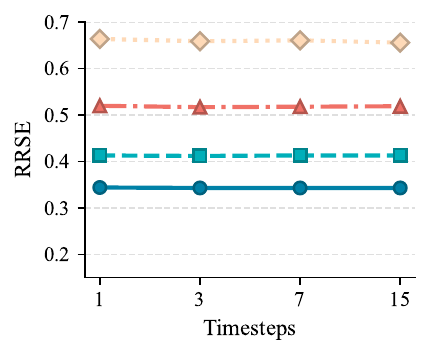}%
  }
  \caption{Ablation study results of different timestep parameters across horizons on Solar and Metr-la datasets.}
  \label{fig:ablation-timesteps} \vspace{-0.5cm}
\end{figure}

\section{Discussion}

\textbf{Neuromorphic Compatibility.}
To assess hardware feasibility, we implemented the core operators within the Lava framework targeting Intel’s Loihi. Going beyond standard theoretical estimates, we conducted a rigorous energy evaluation of our operators on a 22nm process. Unlike prior works that focus solely on computational costs, our analysis explicitly accounts for data movement overheads. By incorporating these memory access dynamics, we provide a fine-grained and more realistic energy profile, ensuring that our reported efficiency gains translate effectively to physical hardware implementations.
Future work will focus on deploying SpikySpace on physical platforms like Loihi~\cite{davies2018loihi} or BrainScaleS-2~\cite{pehle2022brainscales} to obtain empirical latency measurements and validate our simulation assumptions.

\textbf{Forecasting Horizon.}
We observe that SpikySpace performs best at shorter horizons but sees diminishing returns when the horizon is 24. We attribute this to the recursive structure of State-Space Models. Unlike Transformers that use global attention to access the full context, SpikySpace updates its latent state through linear transitions, which can allow errors to accumulate over longer sequences. Nevertheless, the model remains highly valuable for edge computing. For tasks where low latency is critical, such as traffic regulation and anomaly detection, the priority often lies in speed and power efficiency rather than maximizing precision on long horizons.

\section{Conclusion}

In this paper, we presented SpikySpace, a novel neuromorphic framework that adapts State-Space Models for efficient time-series forecasting. SpikySpace addresses the challenge of adapting continuous state-space dynamics for energy-efficient computation by replacing traditional energy-intensive floating-point multiplications with our novel PT-activations and Spiking Selective Scan.
Our experiments demonstrate that this design does not compromise accuracy. Notably, SpikySpace outperforms the dense iTransformer on the Electricity benchmark while reducing energy consumption by over 96\%. This validates that the Spiking Selective Scan mechanism can effectively capture complex temporal dynamics even under sparse, binary precision.

\clearpage
\section*{Impact Statement}
This paper presents work whose goal is to advance the field of machine
learning. There are many potential societal consequences of our work, none
of which we feel must be specifically highlighted here.


\bibliography{ref}
\bibliographystyle{icml2026}

\newpage
\appendix
\onecolumn
\section{Spike Generation}
\label{app:alg}
We provide the detailed spike generation method we adopted in algorithm \ref{algASG}, where $w_{ij}^l$ is the weight of layer $l$ from neuron $i$ to neuron $j$, $b_i^l$ is the bias of the neuron $i$ in layer $l$, $s_i^l$ is the input spike train of the neuron $i$, and $T$ is the time window size.

\begin{algorithm}
	\caption{Average IF model} 
	\label{algASG} 
	\begin{algorithmic}[1]
    
\STATE \textbf{Input:} Weight $w_{ij}^l$ of layer $l$ from neuron $i$ to neuron $j$, bias $b_i^l$, input spike train $s_i^l$, threshold $\theta$, membrane potential $U_i^l(t)$ at timestep $t$, time window size $T$;
\STATE \textbf{Output:} Spike train $s_j$;

\STATE $U^l_i(t) \gets \sum_{j} (w_{ij}^l \cdot s^l_i(t) + b_i^l)$
\STATE $A_i^l \gets \sum_{t=1}^{T}U^l_i(t)/T $

\STATE $s^l_i(0) \gets 0$, $V^l_i(0) \gets 0$;
\FOR {$t \gets 1$ to $T$}

\STATE $V^l_i(t) \gets V^l_i(t-1) + A_i^l$
\STATE $
s^{l+1}_j(t) \gets \begin{cases}
1, & V^l_i(t) \geq \theta\\
0, & \text{otherwise}
\end{cases};
$
\STATE $V^l_i(t) \gets \begin{cases}
V^l_i(t) - \theta, & V^l_i(t) \geq \theta\\
V^l_i(t), & \text{otherwise}
\end{cases}.$ 
\ENDFOR 
\STATE \textbf{return} $s^{l+1}_j$ \COMMENT{Output the ASG spike train of neuron $i$}

\end{algorithmic}
\end{algorithm}

\section{State Space Model}
\label{app:ssm}
A continuous-time linear SSM is given by
\begin{align}
\frac{dx(t)}{dt} &= A\,x(t) + B\,u(t), \\
y(t) &= C\,x(t) + D\,u(t),
\end{align}

where $x(t) \in \mathbb{R}^N$ is the hidden state, $u(t) \in \mathbb{R}^m$ is the input, $y(t) \in \mathbb{R}^p$ is the output, and $A,B,C,D$ are system matrices defining state dynamics and input/output mappings. Discretizing with time step $\Delta t$ yields
\begin{align}
x_{n+1} &= A_d\,x_n + B_d\,u_n, \\
y_n &= C\,x_n + D\,u_n,
\end{align}

where $A_d$ and $B_d$ are the discrete-time equivalents of $A$ and $B$, and $n$ indexes the time step. 
This formulation allows the system to capture temporal dependencies over long sequences: the output $y_n$ can be expressed as a convolution of the input sequence $u$ with an effective kernel
\begin{equation}
y_n = \sum_{k=0}^{n} K_k\, u_{\,n-k}, \quad \text{with } K_k = C\,A_d^k\,B_d.
\end{equation}

\section{Proofs and Analyses}
\subsection{Proof of Lemma~\ref{lem:ptsoftplus_continuous}}\label{sec:ptsoftplus_continuous}

\setcounter{theorem}{0}
\renewcommand{\thetheorem}{4.\arabic{theorem}}

\begin{lem}
    {\rm PTSoftplus} is continuously differentiable.
\end{lem}

\begin{proof}
    Both $2^x$ and $x + C$ are infinitely differentiable on their own. The only point we need to check is $x = x_c$.

    We first show that the function value is continuous at $x = x_c$.
    The left-hand limit of PTSoftplus at $x = x_c$ is
    \[
        \lim_{x \rightarrow x_c^{-}} {\rm PTSoftplus}(x) = \lim_{x \rightarrow x_c^{-}} 2^x = 2^{x_c} = {1\over \ln 2},
    \]
    and the right-hand limit at $x = x_c$ is 
    \[
        \lim_{x \rightarrow x_c^{+}} {\rm PTSoftplus}(x) = \lim_{x \rightarrow x_c^{+}} x + C = x_c + C = {1 \over \ln 2}.
    \]
    The left-hand limit equals the right-hand limit, so the PTSoftplus function is continuous on $\mathbb{R}$.

    Next, we show that the first-order derivative of PTSoftplus is also continuous at $x = x_c$. The derivative of PTSoftplus is as follows:
    \[
        {\rm PTSoftplus} ' (x) = \begin{cases}
            (\ln 2)\cdot 2^x, & \text{if } x < x_c, \\
            1, &\text{otherwise}.
        \end{cases}
    \]
    Since the derivative is a constant when $x \ge x_c$, we only need to check if the left-hand limit of ${\rm PTSoftplus}'$ at $x = x_c$ equals~1. The left-hand limit is given by
    \[
        \lim_{x \rightarrow x_c^{-}} {\rm PTSoftplus}'(x) = \lim_{x \rightarrow x_c^{-}} (\ln 2) \cdot 2^x = (\ln 2) \cdot 2^{x_c}.
    \]
    Substitute $x_c = \log_2({1\over \ln 2})$ into the above. Then the value becomes $(\ln 2) \cdot {1 \over \ln 2} = 1$, which completes the proof. 
\end{proof}

\subsection{Proof of Lemma~\ref{lem:softplus_closeness}}\label{sec:softplus_closeness}
\begin{lem}
    The maximum deviation between the {\rm PTSoftplus} and the {\rm Softplus} function is bounded by 0.914, i.e.
    \[
        \|{\rm PTSoftplus} - {\rm Softplus}\|_\infty  \le 0.914.
    \]
    Moreover, the maximum deviation between the derivatives of {\rm PTSoftplus} and {\rm Softplus} is bounded by 0.371, i.e. 
    \[
        \|{\rm PTSoftplus}' - {\rm Softplus}'\|_\infty \le 0.371.
    \]
\end{lem}

\begin{proof}  
    Define $f(x) = {\rm PTSoftplus}(x) - {\rm Softplus}(x)$. To find an upper bound on the maximum deviation between PTSoftplus and Softplus, it suffices to determine the minimum and maximum values of the function $f$:
    \[
        f(x) = \begin{cases}
            2^x - \ln (1 + e^x), & \text{if } x < x_c, \\
            x - \ln (1 + e^x) + {1\over \ln 2} - \log_2 ({1\over \ln 2}), &\text{otherwise}.
        \end{cases}
    \]
    As $x \rightarrow -\infty$, the value of $f(x)$ approaches 0, and as $x \rightarrow \infty$, the value of $f(x)$ approaches the constant value ${1\over \ln 2} - \log_2 ({1\over \ln 2})$. 
    From Lemma~\ref{lem:ptsoftplus_continuous}, we know that ${\rm PTSoftplus}'$ is continuous, and since ${\rm Softplus}'$ is also continuous, we know that $f$ is continuously differentiable.
    The derivative of $f$ is 
    \[
        f'(x) = \begin{cases}
            (\ln 2) \cdot 2^x - {e^x \over 1 + e^x}, &\text{if } x < x_c, \\
            {1 \over 1 + e^x}, &\text{otherwise}.
        \end{cases}
    \]
    It is clear that, for $x \ge x_c$, the derivative $f'(x)$ is strictly positive, so $f(x)$ is monotonically increasing when $x \ge x_c$. Next, we show that $(\ln 2) \cdot 2^x > {e^x \over 1 + e^x}$ when $x < x_c$, and consequently $f(x)$ is monotonic on $\mathbb{R}$. By the fact that $2^x = e^{x \ln 2}$, the inequality can be reformulated as
    \begin{equation}\label{eq:ptsoftplus_firstorder_ineq}
        (\ln 2) \cdot e^{x \ln 2} > {e^x \over 1 + e^x}.
    \end{equation}
    Divide both sides by $e^{x \ln 2}$. Then the above becomes
    \[
        \ln 2 > {e^{x (1 - \ln 2)} \over 1 + e^x}.
    \]
    Define $g(x)$ as the function on the right-hand side, where
    \[
        g'(x) = {e^{x(1 - \ln 2)}\cdot (1 - \ln 2 - (\ln 2)\cdot e^x) \over (1 + e^x)^2}.
    \]
    Since $(1 + e^x)^2$ and $e^{x(1-\ln 2)}$ are always positive, setting the first derivative $g'(x)$ to 0 gives
    \[
        1 - \ln 2 - (\ln 2) \cdot e^x = 0
        \quad \Rightarrow \quad 
        x = \ln \Big({1 - \ln 2 \over \ln 2}\Big).
    \]
    Denote the critical point $x_g = \ln ({1 - \ln 2 \over \ln 2})$. It can be easily checked that $g'(x) > 0$ when $x < x_g$, and $g'(x) < 0$ when $x > x_g$. Therefore, $g$ is a concave function whose maximum value is
    \[
        g(x_g) = {e^{x_g(1 - \ln 2)} \over 1 + e^{x_g}} \le 0.540 < \ln 2 \ (\approx 0.693),
    \]
    which validates the inequality \eqref{eq:ptsoftplus_firstorder_ineq}. Consequently, $f'(x)$ is strictly positive on $\mathbb{R}$, and $f$ is a monotonic function. Since $f(x) \rightarrow 0$ as $x \rightarrow -\infty$ and $f(x) \rightarrow {1 \over \ln 2} - \log_2({1\over \ln 2})$ as $x \rightarrow \infty$, we conclude that
    \[
        \|f\|_\infty = \sup_{x \in \mathbb{R}} |f(x)| \le {1 \over \ln 2} - \log_2\big({1\over \ln 2}\big) \le 0.914.
    \]
    
    Next, we show that the difference between the derivatives of PTSoftplus and Softplus is also bounded -- that is, the supremum norm of $f'$ is bounded. From the above analyses, we already know that $f'(x)$ is always positive on $\mathbb{R}$. It remains to find an upper bound on the value of $f'(x)$.

    It is clear that for $x \ge x_c$, $f'(x)$ is strictly decreasing. Therefore, on $[x_c, \infty)$, the maximum value of $f'(x)$ is
    \[
        f'(x_c) = {1 \over 1 + e^{x_c}} \le 0.371.
    \]
    For $x < x_c$, the second-order derivative of $f$ is 
    \[
        f''(x) = (\ln 2)^2 \cdot 2^x - {e^x \over (1 + e^x)^2}.
    \]
    In what follows, we show that the above $f''(x) > 0$ on $(-\infty, x_c)$, thus $f'(x)$ is strictly increasing. In other words, we want to show
    \begin{equation}\label{eq:ptsoftplus_secondorder_ineq}
        (\ln 2)^2 > {e^{x(1 - \ln 2)} \over (1 + e^x)^2}.
    \end{equation}
    Define $h(x)$ to be the function on the right-hand side, whose derivative is
    \[
        h'(x) = \frac{e^{x(1-\ln 2)} \cdot (1 + e^x) \cdot [ (1-\ln 2) (1 + e^x) - 2 e^x ] }{ (1 + e^x)^4 }.
    \]
    Since $e^{x(1-\ln 2)} \cdot (1 + e^x)$ is strictly positive, setting $h'(x)$ to zero gives 
    \[
        (1-\ln 2) (1 + e^x) - 2 e^x = 0 \quad \Rightarrow \quad x = \ln \Big( {1 - \ln 2 \over 1 + \ln 2} \Big).
    \]
    Denote the critical point $x_h = \ln ({1 - \ln 2 \over 1 + \ln 2})$. It can be easily checked that $h'(x) > 0$ when $x < x_h$, and $h'(x) < 0$ when $x > x_h$. Therefore, $h$ is a concave function whose maximum value is 
    \[
        h(x_h) = {e^{x_h (1 - \ln 2)} \over (1 + e^{x_h} )^2} \approx 0.424 \le (\ln 2)^2 \ (\approx 0.480),
    \]
    which validates the inequality \eqref{eq:ptsoftplus_secondorder_ineq}. Consequently, $f''(x)$ is strictly positive and $f'(x)$ is increasing on $(-\infty, x_c)$.
    By the fact that $f'(x)$ is continuous, we conclude that
    \[
        \|f'\|_\infty = \sup_{x\in \mathbb{R}} |f'(x)| = f'(x_c) \le 0.371,
    \]
    which completes the proof.
\end{proof}

\subsection{Proof of Lemma~\ref{lem:ptsilu_continuity}}\label{sec:ptsilu_continuity}

\begin{lem}
    {\rm PTSiLU} is continuously differentiable. 
\end{lem}

\begin{proof}
    The power functions and $x$ are all continuously differentiable. So, we only need to check the point~$x = \bar{x}_c$. 

    The left-hand limit of PTSiLU at $x = \bar{x}_c$ is
    \[
        \lim_{x \rightarrow \bar{x}_c^{-}} {\rm PTSiLU}(x) = - 2^{\bar{x}_c} = {1 - \sqrt{1 + 2 \cdot (\ln 2)^2} \over 2 \ln 2},
    \]
    and the right-hand limit of PTSiLU at $x = \bar{x}_c$ is 
    \[
    \begin{aligned}
        \lim_{x \rightarrow \bar{x}_c^{+}}& {\rm PTSiLU}(x) = 2^{-\bar{x}_c - 1} + \bar{x}_c + \bar{C} \\
        &= {\sqrt{1 + 2 \cdot (\ln 2)^2 } + 1 \over 2 \ln 2 } -\frac{\sqrt{1 + 2 \cdot (\ln 2)^2}}{\ln 2} \\
        &= {1 - \sqrt{1 + 2 \cdot (\ln 2)^2} \over 2 \ln 2}.
    \end{aligned}
    \]
    Since the left-hand limit equals the right-hand limit, the function PTSiLU is continuous.

    Next, we show that the derivative of PTSiLU is also continuous. The derivative of PTSiLU is given by
    \[
        {\rm PTSiLU}'(x) = \begin{cases}
            -(\ln 2)\cdot 2^{x}, & \text{if } x < \bar{x}_c, \\
            1 - (\ln 2)\cdot 2^{-x -1}, & \text{otherwise}.
        \end{cases}
    \]
    It is easily seen that the derivative is continuous on $(-\infty, \bar{x}_c)$ and $[x_c, \infty)$. Therefore, we only need to check the continuity of ${\rm PTSiLU}'$ at the transition point $x = \bar{x}_c$.
    The left-hand limit of ${\rm PTSiLU}'(x)$ at $x = \bar{x}_c$ is
    \[
        \lim_{x \rightarrow \bar{x}_c^{-}} {\rm PTSiLU}'(x) = - (\ln 2)\cdot 2^{\bar{x}_c} = {1 - \sqrt{1 + 2 \cdot (\ln 2)^2} \over 2 },
    \]
    and the right-hand limit of ${\rm PTSiLU}'(x)$ at $x = \bar{x}_c$ is
    \[
    \begin{aligned}
        \lim_{x \rightarrow \bar{x}_c^{+}}& {\rm PTSiLU}'(x) = 1 - (\ln 2)\cdot 2^{-\bar{x}_c - 1} \\
        &= 1 - {\sqrt{1 + 2 \cdot (\ln 2)^2 } + 1 \over 2 } \\
        &= {1 - \sqrt{1 + 2 \cdot (\ln 2)^2} \over 2 }.
    \end{aligned}
    \]
    The left-hand limit equals the right-hand limit, so the derivative ${\rm PTSiLU}'$ is also continuous on $\mathbb{R}$. 
\end{proof}

\subsection{Proof of Lemma~\ref{lem:ptsilu_closeness}}\label{sec:ptsilu_closeness}

\begin{lem}
    The maximum deviation between the {\rm PTSiLU} and the {\rm SiLU} function is bounded by 0.316, i.e.
    \[
        \|{\rm PTSiLU} - {\rm SiLU} \|_\infty \le 0.316.
    \]
    Moreover, the maximum deviation between the derivatives of {\rm PTSiLU} and {\rm SiLU} is bounded by $0.263$, i.e.
    \[
        \|{\rm PTSiLU}' - {\rm SiLU}' \|_\infty \le 0.263.
    \]
\end{lem}

\begin{proof}
    Define $\bar{f}(x) = {\rm PTSiLU}(x) - {\rm SiLU}(x)$. Then the analysis of the deviation between PTSiLU and SiLU reduces to the analysis of $\bar f$. We start from the derivatives of $\bar f$.
    The first derivative of $\bar f$ is
    \[
        \bar{f}'(x) = \begin{cases}
             - (\ln2) \cdot 2^x -{e^x\cdot (1 + x + e^x) \over (1 + e^x)^2}, &\text{if } x < \bar{x}_c, \\
            1 - (\ln 2)\cdot 2^{-x-1} - {e^x\cdot (1 + x + e^x) \over (1 + e^x)^2}, &\text{otherwise}.
        \end{cases}
    \]
    On $(-\infty, \bar{x}_c)$, the second derivative of $\bar f$ is 
    \[
        \bar{f}''(x) = - (\ln 2)^2 \cdot 2^x - {e^x (2 + x + e^x (2 - x)) \over (1 + e^x)^3},
    \]
    and on $[\bar{x}_c, \infty)$, the second derivative of $\bar f$ is 
    \[
        \bar{f}''(x) = (\ln 2)^2 \cdot 2^{-x-1} - {e^x (2 + x + e^x (2 - x)) \over (1 + e^x)^3}.
    \]

    Denote $\bar{h}''(x)$ as the function ${e^x (2 + x + e^x (2 - x)) \over (1 + e^x)^3}$, which can be simplified to
    \begin{equation}\label{eq:ptsilu_hx_bound}
    \begin{aligned}
        \bar{h}''(x) = \underbrace{\frac{2e^x}{(1+e^x)^2}}_{A(x)} + \underbrace{\frac{x(e^x - e^{2x})}{(1+e^x)^3}}_{B(x)}.
    \end{aligned}
    \end{equation}
    Here, the first term $A(x)$ is always positive. 
    For the second term $B(x)$, since $(1 + e^x)^3 > 0$, its sign is determined by the numerator $x(e^x - e^{2x})$. 
    When $x < 0$, we have $(e^x - e^{2x}) > 0$, so $B(x) < 0$; when $x > 0$, we have $(e^x - e^{2x}) < 0$, so $B(x)$ is also negative. Therefore, $B(x)$ has a maximum value of 0 at $x = 0$. 
    Since $A(x) > 0$ and $B(x) \le 0$, to bound the values of $\bar{h}''(x)$, we can find extrema of $A(x)$ and $B(x)$ separately. 

    The derivative of $A'(x)$ is 
    \[
        A'(x) = \frac{2(e^x - e^{2x})}{(1+e^x)^3}.
    \]
    Setting it to zero gives $x = 0$. Moreover, $A'(x) > 0$ when $x < 0$, and $A'(x) < 0$ when $x > 0$. Therefore, $A(x)$ attains its maximum value at $x = 0$, where $A(0) = 0.5$. For the second term $B(x)$, it can be checked that $B(x)$ is symmetric (i.e. $B(x) = B(-x)$). Moreover, by the fact that $0 < {1 - e^x  \over (1 + e^x)^2} < 1$ for $x < 0$, we have 
    \[
        B(x) = {x e^x (1 - e^x) \over (1 + e^x)^3} \ge xe^x \quad \text{for } x < 0.
    \]
    The function $-xe^x$ attains its minimum value of $-{1\over e}$ at $x = -1$. Therefore, we have $B(x) \ge -{1\over e} \ge -0.368$. Together with \eqref{eq:ptsilu_hx_bound}, we have $\|\bar{h}''\|_\infty \le \max\{ \sup |A(x)|, \sup |B(x)| \} = 0.5$. Therefore, for $\bar{f}''$, since $-(\ln 2)^2 \cdot 2^x$ is strictly decreasing on $(-\infty, \bar{x}_c)$, we have 
    \[
        \sup_{x \in (-\infty, \bar{x}_c) } |\bar{f}''(x)| \le (\ln 2)^2 \cdot 2^{\bar{x}_c} + \|\bar{h}''\|_\infty \le 0.639.
    \]
    On the other hand, since $(\ln 2)^2 \cdot 2^{-x-1}$ is strictly decreasing on $[\bar{x}_c , \infty)$, we have
    \[
        \sup_{x \in [\bar{x}_c, \infty) } |\bar{f}''(x)| \le (\ln 2)^2 \cdot 2^{-\bar{x}_c-1} + \|\bar{h}''\|_\infty \le 1.332.
    \]
    Combining the above two yields $\|\bar{f}''\|_\infty \le 1.332$. In other words, the rate of change of $\bar{f}'(x)$ is bounded by 1.332. 

    Based on the above bound on $\|\bar{f}''\|_\infty$, we employ a computer aided analysis for $\bar{f}'$ and $\bar{f}$. We generate a sequence of points ${\cal X} = \{x_1, x_2, \dots, x_N\}$ from -10 to 10, where $x_{i+1} - x_i = 0.001$ for $i = 1, \dots, N-1$. Then we compute the value of $\bar{f}'(x)$ for each $x \in {\cal X}$. The numerical results show that these values are between -0.178 and 0.261. By the fact that $\|\bar{f}''\|\le 1.332$, we have
    \[  
        \sup_{x \in [-10, 10]} |\bar{f}'(x)| \le 0.261 + 1.332 \cdot 0.001 \le 0.263.
    \]
    It can be shown that $\bar{f}'(x)$ diminishes when $x \notin [-10, 10]$.

    Next, we evaluate the value of $\bar{f}(x)$ for each $x \in {\cal X}$. The numerical results show that these values are between -0.229 and 0.315. By the fact that $\|\bar{f}'\|\le 0.263$, we have
    \[
        \sup_{x \in [-10, 10]} |\bar{f}(x)| \le 0.315 + 0.263 \cdot 0.001 \le 0.316.
    \]
    It is easily seen that $\bar{f}(x)$ is negligible when $x < -10$. For $x > 10$, the function value $\bar{f}(x)$ eventually decreases to the limit $\lim_{x \rightarrow \infty} \bar{f}(x) \ge - 0.229$.
\end{proof}

\section{Detailed Experimental Settings}
\label{app:exp_details}

\subsection{Datasets}
\label{app:dataset}
In this section, we provide a detailed description of the dataset used in our experiments. All datasets are divided chronologically into training, validation, and test sets without shuffling. 
The statistical details and split ratios are summarized in Table~\ref{tab:dataset_stats}.

\begin{table}[h]
\centering
\caption{Statistics of the four benchmarks.}
\label{tab:dataset_stats}
\setlength{\tabcolsep}{5pt}
\renewcommand{\arraystretch}{1.1}
\begin{tabular}{lcccc}
\toprule
\textbf{Dataset} & \textbf{\# Samples} & \textbf{\# Variables} & \textbf{Sample Rate} & \textbf{Length}\\
\midrule
Metr-la     & 34,272  & 207 & 5 minutes  & 12   \\
Pems-bay    & 52,116  & 325 & 5 minutes  & 12   \\
Solar       & 52,560  & 137 & 10 minutes & 168  \\
Electricity & 26,304  & 321 & 1 hour     & 168  \\
\bottomrule
\end{tabular}%
\end{table}

\begin{itemize}
    \item \textbf{METR-LA} \cite{li2018diffusion}: 
    Traffic speed data collected from 207 loop detectors on Los Angeles highways, aggregated every 5 minutes.
    The dataset contains 34,272 samples, each with a sequence length of 12 and 207 variables.
    It is split into training, validation, and test sets with a ratio of $(0.7,\,0.2,\,0.1)$.
    This dataset captures strong spatial–temporal dependencies due to road network topology.

    \item \textbf{PEMS-BAY} \cite{li2018diffusion}:
    Another large-scale traffic dataset from the California Bay Area, sampled every 5 minutes from 325 sensors.
    It includes 52,116 samples, each of length 12 with 325 variables.
    The chronological split ratio is $(0.7,\,0.2,\,0.1)$.
    PEMS-BAY is widely adopted for benchmarking spatiotemporal forecasting models.

    \item \textbf{Solar-Energy} \cite{lai2018modeling}:
    Solar power production recorded from 137 photovoltaic plants across several years, exhibiting strong daily and seasonal periodicity.
    The dataset contains 52,560 hourly samples with 137 variables and a sequence length of 168.
    It is partitioned into training, validation, and test sets with a ratio of $(0.6,\,0.2,\,0.2)$.

    \item \textbf{Electricity} \cite{lai2018modeling}:
    Hourly electricity consumption data from 321 clients over multiple years, showing clear diurnal and weekly patterns.
    The dataset includes 26,304 samples, each with 321 variables and a sequence length of 168.
    It is divided into train/validation/test sets in the ratio $(0.6,\,0.2,\,0.2)$.
\end{itemize}

\subsection{Baselines}
We compare SpikySpace with the following methods:
\begin{itemize}
    \item GRU \cite{cho2014learning}: A recurrent neural network (RNN) model that captures temporal dependencies through gated recurrent units. It serves as a classical ANN baseline for sequence modeling.

    \item iTransformer \cite{liuitransformer}: A state-of-the-art transformer variant that learns both instance-wise and channel-wise dependencies in multivariate time series, representing strong ANN performance under full-precision computation.

    \item SpikeTCN \cite{lv2024efficient}: A convolutional SNN that integrates temporal convolution with spike-based processing, highlighting the benefit of temporal feature extraction in event-driven computation.

    \item SpikeRNN \cite{lv2024efficient}: A recurrent SNN model that replaces analog activations with spiking neurons, enabling energy-efficient sequential modeling through temporal dynamics.

    \item iSpikformer \cite{lv2024efficient}: A spiking adaptation of the transformer architecture that replaces softmax attention and linear layers with spike-compatible modules, achieving competitive accuracy with reduced energy.

    \item SpikeSTAG \cite{hu2025spikestag}: A hybrid spatio-temporal graph SNN designed for structured time-series data, combining graph connectivity with event-driven updates to model spatial correlations efficiently.
    
\end{itemize}

\subsection{Evaluation Metrics}
For evaluation, we reported the coefficient of determination ($R^2$) and Root Relative Squared Error (RRSE) for each dataset across all the horizons. Let $y_t$ denote the ground-truth value at time step $t$, $\hat{y}t$ denote the corresponding model prediction, and $\bar{y}$ denote the mean of all ground-truth values. The two metrics are defined as:
\begin{equation}
    R^2 = 1 - \frac{\sum_{t} (y_t - \hat{y}_t)^2}{\sum_{t} (y_t - \bar{y})^2},
    \label{eq:r2}
\end{equation}
\begin{equation}
    \mathrm{RRSE} = 
    \sqrt{
        \frac{\sum_{t} (y_t - \hat{y}_t)^2}
             {\sum_{t} (y_t - \bar{y})^2}
    },
    \label{eq:rrse}
\end{equation}
A higher $R^2$ indicates stronger goodness of fit, while a lower RRSE reflects smaller relative error and better cross-dataset comparability.
We follow prior work in adopting these two metrics rather than MAE or MSE, since they are scale-invariant and thus comparable across datasets with different magnitudes, and they better highlight both relative deviation and variance-explained quality.

\subsection{Implementation Details}
In all experiments, we set the batch size to 64 and the learning rate to $5\times10^{-4}$. 
For each dataset, the forecasting horizons are set to 3, 6, 12, and 24 time steps. 
Our model is trained using the Adam optimizer with the mean squared error (MSE) as the loss function. 
During training, we set the maximum epoch number as 1000 while the early stop patience as 20.
All experiments are conducted on a single NVIDIA A100 GPU with 80GB of memory. The SpikySpace model results we reported in Table~\ref{tab:forecast_comparison} are using a timestep $T=3$, which means we use a spike train length of 3 to represent a scalar.

\subsection{Quantization Details}
\label{app:taq}
The straight-through estimator (STE) is employed as commonly used in previous works to back-propagate the gradients to $\alpha$. 

Our quantization method covers both symmetric and asymmetric cases. 
Specifically, the real-valued tensor $X_R$ is first normalized by subtracting the offset $\beta$ and dividing by $\alpha$. 
The result is rounded to the nearest integer and clipped to the representable integer range $[Q_n, Q_p]$, 
which is determined by the bit width $b$. 
Finally, the quantized value is mapped back to the real domain as $X_Q$. 
When $\beta = 0$ and $(Q_n, Q_p) = (-2^{b-1},\, 2^{b-1}-1)$, 
the formulation reduces to symmetric quantization.
When $\beta \neq 0$ and $(Q_n, Q_p) = (0,\, 2^b-1)$, 
It becomes asymmetric quantization.


\end{document}